\documentclass[letterpaper]{article} 
\usepackage{aaai2026}  
\usepackage{times}  
\usepackage{helvet}  
\usepackage{courier}  
\usepackage[hyphens]{url}  
\usepackage{graphicx} 
\usepackage{natbib}  
\usepackage{caption} 
\usepackage{amsmath,amsthm,amssymb,amscd}

\newtheorem{def.}{Definition}
\newtheorem{prop.}{Proposition}
\newtheorem{coro.}{Corollary}
\newtheorem{exam.}{Example}
\newtheorem{theorem}{Theorem}
\newtheorem{remark}{Remark}
\title{A Sheaf-Theoretic and Topological Perspective on Complex Network Modeling and Attention Mechanisms in Graph Neural Models}

\title{A Sheaf-Theoretic and Topological Perspective on Complex Network Modeling and Attention Mechanisms in Graph Neural Models}
\author {
    Chuan-Shen Hu
}
\affiliations{
    Department of Mathematics, National Central University \\
    No. 300, Zhongda Rd., Zhongli District, Taoyuan City 320317, Taiwan\\
    chuanshenhu.official@gmail.com
}

\usepackage{bibentry}

\begin{document}

\maketitle

\begin{abstract}
Combinatorial and topological structures, such as graphs, simplicial complexes, and cell complexes, form the foundation of geometric and topological deep learning (GDL and TDL) architectures. These models aggregate signals over such domains, integrate local features, and generate representations for diverse real-world applications. However, the distribution and diffusion behavior of GDL and TDL features during training remains an open and underexplored problem. Motivated by this gap, we introduce a cellular sheaf theoretic framework for modeling and analyzing the local consistency and harmonicity of node features and edge weights in graph-based architectures. By tracking local feature alignments and agreements through sheaf structures, the framework offers a topological perspective on feature diffusion and aggregation. Furthermore, a multiscale extension inspired by topological data analysis (TDA) is proposed to capture hierarchical feature interactions in graph models. This approach enables a joint characterization of GDL and TDL architectures based on their underlying geometric and topological structures and the learned signals defined on them, providing insights for future studies on conventional tasks such as node classification, substructure detection, and community detection.
\end{abstract}


\section{Introduction}

End-to-end geometric and topological deep learning (GDL and TDL) models, including graph neural networks (GNNs) and simplicial neural networks (SNNs), are built upon combinatorial, geometric, and topological structures such as graphs, simplicial complexes, and cell complexes~\cite{GDL_Bronstein,hajij2023topologicaldeeplearninggoing,zia2024topological}. These frameworks have become powerful and versatile tools for addressing a wide range of real-world applications. Notably, graph structure-based deep learning models, such as the graph convolutional network (GCN) and graph attention network (GAT) models~\cite{KipfGCN2017,velickovic2018graph}, have demonstrated remarkable effectiveness in analyzing social, biological, and scientific data with complex geometric structures, such as social networks~\cite{LI2023126441}, molecular graphs~\cite{NIPS2015_f9be311e,Kearnes2016,chen2019graph}, and citation networks~\cite{Gong_2019_CVPR,Xie_GNN_citation_2021,Jiang_2019_CVPR}. 

Furthermore, simplicial neural networks (SNNs), or more generally, cell complex neural networks (CXNs), extend GNN architectures by incorporating higher-dimensional objects such as $1$-, $2$-, or $3$-simplices (i.e., edges, triangles, and tetrahedra) as higher-order nodes~\cite{ebli2020simplicialneuralnetworks,hajij2021cellcomplexneuralnetworks}. These models perform feature diffusion and aggregation based on face relations across simplices of different dimensions, enabling the capture of richer geometric and topological interactions within complex structures~\cite{pmlr-v139-bodnar21a}.
 
From a mathematical perspective, GNN, SNN, and CXN models primarily implement neighborhood-based feature diffusion frameworks, where neuron interactions are determined primarily by the topological structure of the underlying domain, such as (co)facial, upper, and lower adjacencies among simplices or cells~\cite{ebli2020simplicialneuralnetworks,hajij2021cellcomplexneuralnetworks}. However, the relationship between the topology of local regions and the embedding features (or signals) defined on them requires a more refined investigation. In particular, GDL and TDL models that explicitly account for the local alignment and harmonicity of node features along edges---especially in combinatorial or graph-structured data with geometric or physical relevance~\cite{short2022current,cooperband2023towards,cooperband2023cosheaftheoryreciprocalfigures,cooperband2024equivariantcosheavesfinitegroup,cooperband2025unifiedorigamikinematicscosheaf}---remain largely unexplored. 

Sheaf theory~\cite{Tennison_1975,BredonSheaf1997}, a cornerstone of modern mathematics, provides a unified framework for jointly modeling both a geometric object and the signals defined on its local regions, known as \textit{local sections}. It analyzes how these signals behave on overlaps of local regions and relates them to the global structure and its corresponding \textit{global sections}. In recent years, \textit{cellular sheaves}~\cite{Michael_Robinson_book_2014,curry2014sheaves,Curry2015,hansen2019toward,arya2025sheaf}, a discretized formulation of sheaves defined on combinatorial structures such as graphs, simplicial complexes, and cell complexes, have gained increasing attention in GDL and TDL architecture design~\cite{hansen_shNN_2020,barbero2022sheaf,bodnar2022neural,barbero2022sheafLap}. Cellular sheaves offer an elegant mathematical foundation for representing feature dependencies, enforcing local consistency, and modeling information flow across structured domains. 

In particular, by integrating the sheaf Laplacian into the diffusion process, these methods have been instrumental in advancing sheaf neural networks (ShNNs), sheaf attention networks (SheafANs), and other topology-aware learning models, which have shown promising results in applications such as social network analysis and citation network modeling~\cite{barbero2022sheaf,bodnar2022neural,barbero2022sheafLap}. Notably, SheafAN employs a sheaf–attention mechanism that leverages directional weights between nodes in the underlying graph, using learnable cellular sheaves parameterized by MLPs and incorporating the corresponding sheaf Laplacian into the diffusion algorithms~\cite{barbero2022sheaf}.

Rather than directly incorporating the sheaf Laplacian into the feature diffusion process of graph- or simplicial complex–based models, this paper investigates the potential of leveraging the sheaf structures naturally induced by the signals, weights, and topology of the underlying domain---specifically through local sections and the Alexandrov topology on graphs~\cite{alexandroff1937diskrete}---to analyze the training dynamics and geometric interpretation of signal behavior within GDL and TDL architectures. In particular, we focus on a theoretical model of the core attention mechanism in GATs through a cellular sheaf representation, examining its underlying topological and geometric interpretations. We further introduce the notion of a harmonic substructure of a sheaf on a graph, which quantifies how harmonically node signals are distributed with respect to a given sheaf structure. Moreover, from a topological data analysis (TDA) perspective~\cite{carlsson2004persistence,zomorodian2004computing,cohen2005stability,ghrist2008barcodes,carlsson2009topology}, we propose a multiscale framework for evaluating the harmonicity of sheaf structures. This perspective provides new insights and analytical tools for understanding and assessing tasks such as node classification, substructure detection, and community detection.

\section{Main Results}
The main results of this paper can be summarized in three parts: (a) the identification of the core attention mechanism in GAT models---comprising node features and learned attention weights---as a cellular sheaf on the given graph; (b) the harmonicity analysis of cellular sheaf structures defined on graphs; and (c) the development of a TDA-based multiscale framework for evaluating the harmonicity of cellular sheaves. 

Result (a) is presented in Theorem~\ref{Theorem: Main result 1}, which shows that any GAT triple (see Equation~\eqref{Eq. The GAT triple}) can be regarded as a cellular sheaf on the underlying graph. It is worth noting that the proposed cellular sheaf structure differs from that of SheafAN~\cite{barbero2022sheaf}. In~\cite{barbero2022sheaf}, the cellular sheaf is trained by learning the restriction linear maps and is integrated into the message-passing process alongside the attention weights for updating node features. In contrast, in our framework, the trained attention weights of the GAT model themselves define a cellular sheaf, and we investigate how these weights induce and characterize the corresponding sheaf structure. 

Result (b) is primarily presented in Definition~\ref{Definition: Main result 2-1, harmonic sets}, Theorem~\ref{Theorem: Main result 2-2}, and Corollary~\ref{Corollary: Main result 2-3}. Specifically, for a given cellular sheaf defined on a finite, undirected, simple, and unweighted graph with node feature vectors, we introduce the notions of harmonic node sets and harmonic edge sets. These harmonic structures quantify the degree of local alignment among node signals. 

Result (c) is presented in Theorem~\ref{Theorem: Main result 3-1} and Corollary~\ref{Corollary: Main result 3-2}, which establishes a TDA-oriented multiscale framework for evaluating the harmonicity of cellular sheaf structures. By constructing a \textit{filtration} of graphs based on result (b), this framework captures the evolution of local-to-global alignments among node signals across multiple scales. In particular, as a future direction, the resulting TDA summary, such as the \emph{persistence barcode} (see, for example, \cite{carlsson2004persistence,ghrist2008barcodes}), encodes the scale-dependent behavior of harmonic substructures, providing a quantitative characterization of cellular and node features on graphs.

\section{Mathematical Background}

\subsection{Conventions and Terminologies} 
In this paper, a \textit{graph} refers to a finite, undirected, simple, and unweighted graph. To identify a graph as a \textit{partially ordered set} (or \textit{poset}, for simplicity), we regard a graph $G = V \cup E$ as the union of its node set $V$ and edge set $E$. For a node $v \in V$ and an edge $e \in E$, the relation $v \unlhd e$ holds if and only if $v$ is an endpoint of $e$. This relation $\unlhd$ on $G$ is called the \textit{face relation}. Each edge $e \in E$ therefore has exactly two distinct nodes $v, w \in V$ such that $v \unlhd e$ and $w \unlhd e$. In addition, the reflexive relations $v \unlhd v$ and $e \unlhd e$ hold for every node $v \in V$ and edge $e \in E$. Hence, the pair $(G,\unlhd)$ forms a poset. The empty set $\emptyset$ is also regarded as a graph, known as the \textit{null graph},  which contains neither vertices nor edges. 

To define a cellular sheaf on a graph, this paper also employs the necessary elementary notions and terminologies from category theory. In particular, a graph $(G,\unlhd)$ can be identified with a \textit{poset category}, whose objects are the vertices and edges of $G$, and whose morphisms (or arrows) correspond to the partial order relations between these elements. On the other hand, this paper also considers the category of real vector spaces and linear transformations, denoted by $\text{Vect}_{\mathbb{R}}$, which consists of all $\mathbb{R}$-vector spaces as objects and $\mathbb{R}$-linear maps as morphisms.

\begin{figure}[t]
\centering
\includegraphics[width=0.75\columnwidth]{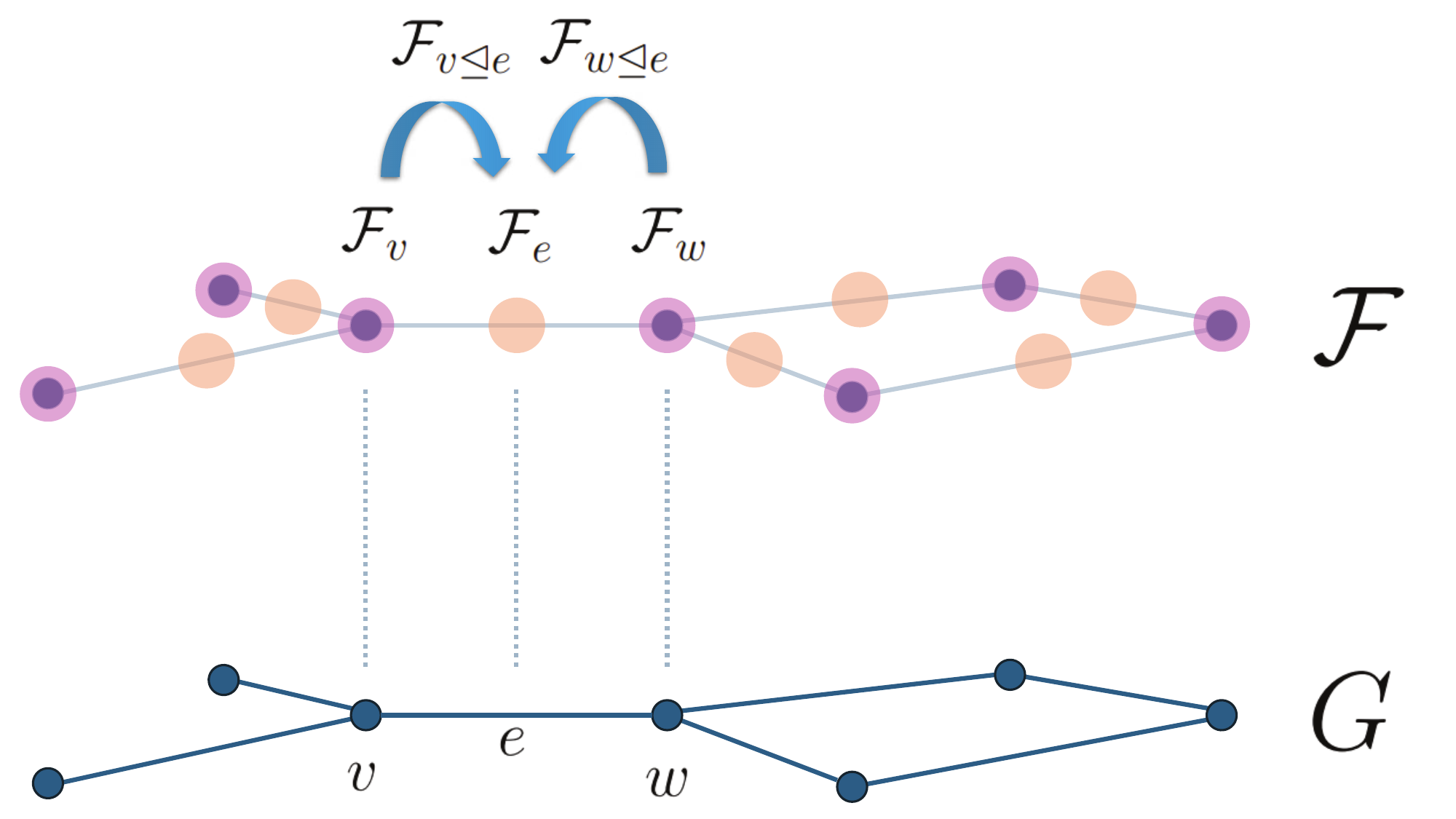} 
\caption{Illustration of a cellular sheaf $\mathcal{F}: (G,\unlhd) \to \text{Vect}_{\mathbb{R}}$ defined on a graph $G$. Disks on the top layer represent the stalk spaces assigned to the nodes and edges, whose dimensions may vary. In this example, two nodes $v$ and $w$, together with an edge $e$, are highlighted along with their corresponding stalk spaces $\mathcal{F}_v$, $\mathcal{F}_w$, and $\mathcal{F}_e$, as well as the restriction maps $\mathcal{F}_{v \unlhd e}$ and $\mathcal{F}_{w \unlhd e}$.}
\label{Fig_Illustration}
\end{figure}

\subsection{Cellular Sheaves on Graphs}

Viewing the graph $(G,\unlhd)$ as a poset category, a \textit{cellular sheaf} of $\mathbb{R}$-vector spaces on $G$ is defined as a functor $\mathcal{F} : (G,\unlhd) \to \text{Vect}_{\mathbb{R}}$. Equivalently, it consists of the following data: (a) a \textit{stalk space} $\mathcal{F}_v$ for each vertex $v \in V$; (b) a \textit{stalk space} $\mathcal{F}_e$ for each edge $e \in E$; and (c) a \textit{restriction map} $\mathcal{F}_{v,e} : \mathcal{F}_v \to \mathcal{F}_e$ for each pair $(v,e) \in V \times E$ satisfying $v \unlhd e$. Conventionally, when representing the coboundary matrix of a cellular sheaf $\mathcal{F}$ (see Equation~\eqref{Eq. Sheaf coboundary matrix}), we adopt the convention that $\mathcal{F}_{v,e}: \mathcal{F}_v \to \mathcal{F}_e$ is defined as the zero map whenever $v$ is not an endpoint of $e$. In particular, when the relation $v \unlhd e$ holds, the map $\mathcal{F}_{v,e}$ is also denoted by $\mathcal{F}_{v \unlhd e}$ to emphasize that $v$ is an endpoint of $e$. Figure~\ref{Fig_Illustration} presents a graphical illustration of a cellular sheaf defined on a graph, depicting both the underlying graph structure and the associated signal data defined on its nodes and edges.

\begin{exam.}[Constant Sheaf]
Let $(G, \unlhd)$ be a graph with node and edge sets $V$ and $E$, and let $W$ be an $\mathbb{R}$-vector space.  
The \textbf{constant sheaf} associated with $W$ over $G$, denoted by $\underline{W}$, is the cellular sheaf of $\mathbb{R}$-vector spaces on $G$ defined as follows:
\begin{itemize}
    \item $\underline{W}_v = W$ for every $v \in V$;
    \item $\underline{W}_e = W$ for every $e \in E$;
    \item $\underline{W}_{v \unlhd e} = \textup{\text{id}}_W = \underline{W}_{w \unlhd e}$ for each edge $e = \{v, w\} \in E$.
\end{itemize}
Here, the function $\textup{\text{id}}_W: W \to W$ denotes the identity map on the vector space $W$. By convention, $\underline{W}_{v,e}: \underline{W}_v \to \underline{W}_e$ denotes the zero map if $v$ is not an endpoint of $e$.
\end{exam.}

The constant sheaf on a graph assigns the same vector space to all nodes and edges and transfers node signals to adjacent edges via the identity map on that vector space. This construction encapsulates the entire graph structure by treating nodes and edges uniformly within the same signal space. In particular, it is a well-known result that the standard graph Laplacian coincides with the sheaf Laplacian of the constant sheaf $\underline{\mathbb{R}}$ (see Equation~\eqref{Eq. Sheaf Laplacian}).

\subsection{Sheaf Laplacian and Global Section Space}

To quantify the consistency of node signals of a graph encoded by a cellular sheaf $\mathcal{F}: (G,\unlhd) \to \text{Vect}_{\mathbb{R}}$, several fundamental constructions are essential. These include the $0$th and $1$st sheaf cochain spaces $C^0(G;\mathcal{F})$ and $C^1(G;\mathcal{F})$, the $0$th sheaf coboundary matrix $\mathbf{C}^0_{\mathcal{F}}: C^0(G;\mathcal{F}) \to C^1(G;\mathcal{F})$, the global section space $\Gamma(G;\mathcal{F})$, and the $0$th sheaf Laplacian $\mathbf{L}^0_{\mathcal{F}}: C^0(G;\mathcal{F}) \to C^0(G;\mathcal{F})$. We briefly review these fundamental notions below.

Specifically, for a cellular sheaf $\mathcal{F}: (G,\unlhd) \to \text{Vect}_{\mathbb{R}}$ on a graph $G = V \cup E$, the $0$th and $1$st \textit{sheaf cochain spaces} $C^0(G;\mathcal{F})$ and $C^1(G;\mathcal{F})$ are defined as the $\mathbb{R}$-vector spaces given by the direct sums
\begin{equation*}
C^0(G;\mathcal{F}) = \bigoplus_{v \in V} \mathcal{F}_v 
\quad \text{and} \quad 
C^1(G;\mathcal{F}) = \bigoplus_{e \in E} \mathcal{F}_e,
\end{equation*}
where elements of $C^0(G;\mathcal{F})$ and $C^1(G;\mathcal{F})$ are typically represented as the tuples $\mathbf{s} = (\mathbf{s}_v)_{v \in V}$ and $\mathbf{t} = (\mathbf{t}_e)_{e \in E}$, with $\mathbf{s}_v \in \mathcal{F}_v$ and $\mathbf{t}_e \in \mathcal{F}_e$ for each $v \in V$ and $e \in E$, respectively. Elements in $C^0(G; \mathcal{F})$ and $C^1(G; \mathcal{F})$ are called the $0$th and $1$st \textit{sheaf cochains} of the cellular sheaf $\mathcal{F}$. In particular, since $G$ is finite, each sheaf cochain $\mathbf{s} = (\mathbf{s}_v)_{v \in V}$ or $\mathbf{t} = (\mathbf{t}_e)_{e \in E}$ can be represented as column vectors, whose block entries correspond to the component vectors $\mathbf{s}_v \in \mathcal{F}_v$ and $\mathbf{t}_e \in \mathcal{F}_e$, respectively.

Suppose $G = V \cup E$ consists of $n$ nodes and $m$ edges, with vertex set $V = \{ v_1, \ldots, v_n \}$ and edge set $E = \{ e_1, \ldots, e_m \}$. The $0$th \textit{sheaf coboundary matrix} $\mathbf{C}^0_{\mathcal{F}}: C^0(G;\mathcal{F}) \to C^1(G;\mathcal{F})$ is defined as follows. 
By fixing a total order of the nodes $v_1 < v_2 < \cdots < v_n$, each edge $e \in E$ can be uniquely represented as an ordered pair $[v,w]$, where $v,w \in V$ are the endpoints of $e$ with $v < w$. 
The corresponding \textit{sign-incidence indices} $[v:e]$ and $[w:e]$ are then defined by
\begin{equation*}
[v:e] = -1 \quad \text{and} \quad [w:e] = 1.
\end{equation*}
If $u \in V$ is not an endpoint of $e \in E$, the sign-incidence index $[u:e]$ is conventionally set to $0$. With these settings, the $0$th sheaf coboundary matrix $\mathbf{C}^0_{\mathcal{F}}: C^0(G;\mathcal{F}) \to C^1(G;\mathcal{F})$ can be expressed as the block matrix
\begin{equation}\label{Eq. Sheaf coboundary matrix}
\mathbf{C}^0_{\mathcal{F}} = \begin{bmatrix}
[v_1:e_1] \mathcal{F}_{v_1,e_1} & \cdots & [v_n:e_1]  \mathcal{F}_{v_n,e_1} \\
[v_1:e_2] \mathcal{F}_{v_1,e_2} & \cdots & [v_n:e_2] \mathcal{F}_{v_n,e_2} \\
\vdots & \ddots & \vdots \\
[v_1:e_m] \mathcal{F}_{v_1,e_m} & \cdots & [v_n:e_m] \mathcal{F}_{v_n,e_m} 
\end{bmatrix}.
\end{equation}
In particular, for every vector $\mathbf{s} = (\mathbf{s}_v)_{v \in V} \in C^0(G;\mathcal{F})$, the vector $\mathbf{C}^0_{\mathcal{F}}\mathbf{s} = \mathbf{t} = (\mathbf{t}_e)_{e \in E} \in C^1(G;\mathcal{F})$ 
can be expressed as the column vector
\begin{equation*}
\mathbf{C}^0_{\mathcal{F}} \cdot \begin{bmatrix} 
\mathbf{s}_{v_1} \\
\mathbf{s}_{v_2} \\
\vdots \\
\mathbf{s}_{v_n} \\
\end{bmatrix} = \begin{bmatrix} 
\mathbf{t}_{e_1} \\
\mathbf{t}_{e_2} \\
\vdots \\
\mathbf{t}_{e_m} \\
\end{bmatrix} = \begin{bmatrix} 
\sum_{i=1}^n [v_i:e_1] \mathcal{F}_{v_i,e_1} \mathbf{s}_{v_i} \\
\sum_{i=1}^n [v_i:e_2] \mathcal{F}_{v_i,e_2} \mathbf{s}_{v_i} \\
\vdots \\
\sum_{i=1}^n [v_i:e_m] \mathcal{F}_{v_i,e_m} \mathbf{s}_{v_i} \\
\end{bmatrix}  
\end{equation*}

More precisely, since each edge has exactly two endpoints, every block row of the matrix $\mathbf{C}^0_{\mathcal{F}}$ contains at most two nonzero blocks. For each edge $e_l$ ($l \in \{1, 2, \ldots, m\}$) with endpoints $v_j < v_k$ ($j, k \in \{1, 2, \ldots, n\}$), the $l$th block of the vector $\mathbf{t}$ is given by
\begin{equation}\label{Eq. Sheaf coboundary block}
\mathbf{t}_{e_l} = \mathcal{F}_{v_k, e_l} \mathbf{s}_{v_k} - \mathcal{F}_{v_j, e_l} \mathbf{s}_{v_j}.
\end{equation}
Specifically, $\mathbf{t}_{e_l}$ is the zero vector if and only if the transferred edge signals 
$\mathcal{F}_{v_k, e_l} \mathbf{s}_{v_k}$ and $\mathcal{F}_{v_j, e_l} \mathbf{s}_{v_j}$ are identical. This observation explains the definition of the \textit{global section space} 
$\Gamma(G;\mathcal{F})$ of $\mathcal{F}$, which consists of all vectors $(\mathbf{s}_v)_{v \in V} \in C^0(G;\mathcal{F})$ satisfying Equation~\eqref{Eq. Sheaf coboundary block} with $\mathbf{t}_{e_l} = \mathbf{0}$ for all $l$; these vectors are called the \textit{global sections} of $\mathcal{F}$. That is, the global section space is the kernel of the sheaf coboundary matrix $\mathbf{C}^0_{\mathcal{F}}$, i.e., $\Gamma(G;\mathcal{F}) = \ker(\mathbf{C}^0_{\mathcal{F}})$.

With the $0$th sheaf coboundary matrix $\mathbf{C}^0_{\mathcal{F}}$ of a cellular sheaf $\mathcal{F}: (G,\unlhd) \to \text{Vect}_{\mathbb{R}}$, the $0$th \textit{sheaf Laplacian matrix} of $\mathcal{F}$ is defined as
\begin{equation}\label{Eq. Sheaf Laplacian}
\mathbf{L}^0_\mathcal{F} = (\mathbf{C}^0_{\mathcal{F}})^\top \mathbf{C}^0_{\mathcal{F}},
\end{equation}
which acts as a linear operator on the $0$th sheaf cochain space $C^0(G;\mathcal{F})$.
By definition, $\mathbf{L}^0_\mathcal{F}$ is symmetric and positive semi-definite. Consequently, it is diagonalizable, and all its eigenvalues are non-negative. Therefore, there exists an $\mathbb{R}$-basis of $C^0(G;\mathcal{F})$ consisting of eigenvectors of $\mathbf{L}^0_\mathcal{F}$. In particular, any node signal vector $\mathbf{s} \in C^0(G;\mathcal{F})$ can be uniquely expressed as a linear combination of these eigenvectors. Furthermore, by the Gram–Schmidt algorithm, the $\mathbb{R}$-basis can be chosen to be orthonormal.

The global section space and the sheaf Laplacian matrix of a cellular sheaf on a graph are closely related. Specifically, by the well-known application of the combinatorial Hodge decomposition theorem (see, for example,~\cite{hansen2019toward}) to the (co)chain complexes of finite-dimensional $\mathbb{R}$-vector spaces, one obtains
\begin{equation}\label{Eq. Hodge decomposition theorem}
\ker(\mathbf{L}_{\mathcal{F}}^0) = \ker(\mathbf{C}_{\mathcal{F}}^0) = \Gamma(G;\mathcal{F}).
\end{equation}
In other words, a vector $\mathbf{s} \in C^0(G;\mathcal{F})$ is a global section of $\mathcal{F}$ if and only if it is an eigenvector of the sheaf Laplacian matrix $\mathbf{L}_{\mathcal{F}}^0$ corresponding to the eigenvalue~$0$.

\begin{figure}[t]
\centering
\includegraphics[width=0.85\columnwidth]{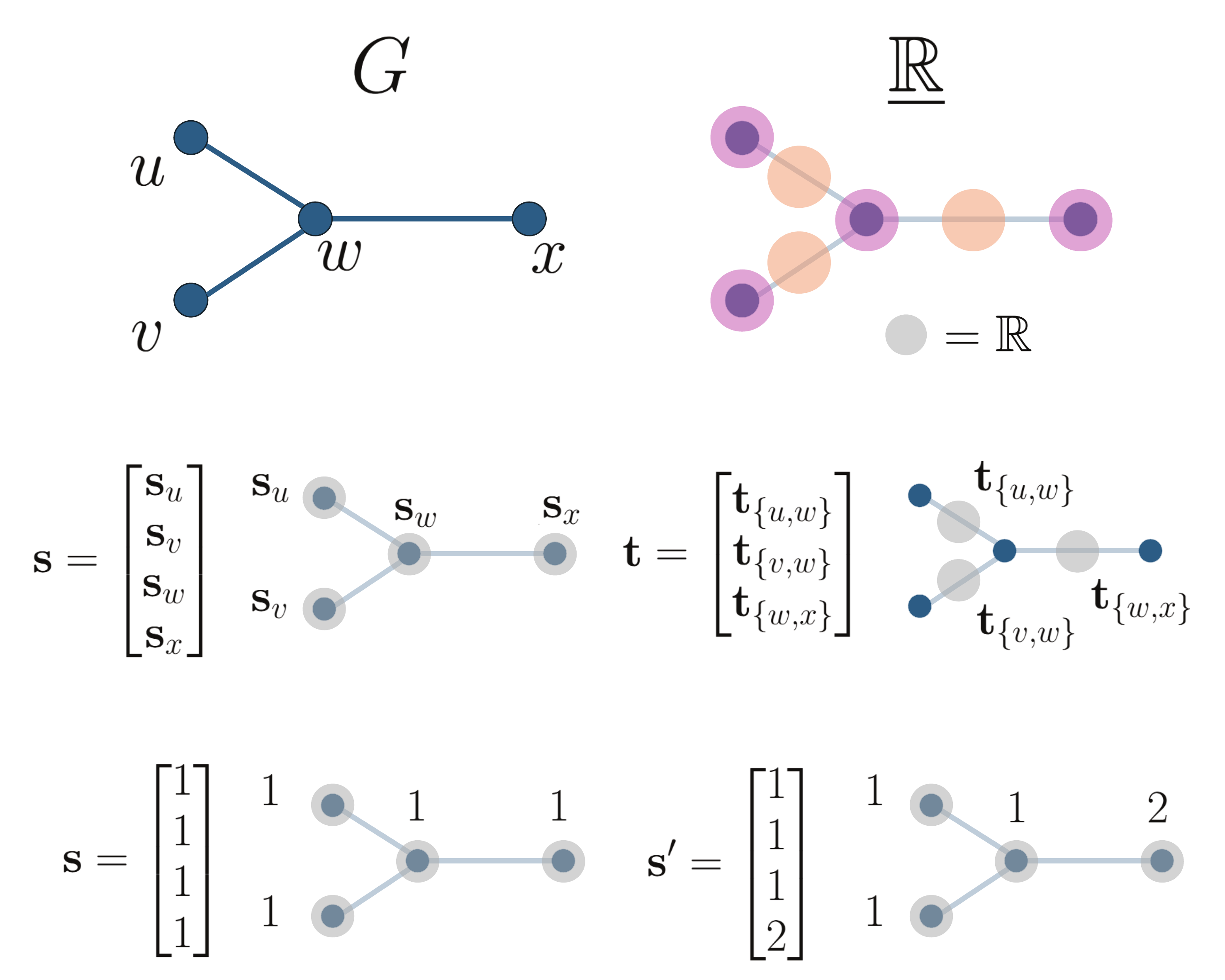} 
\caption{An illustrative example of a graph $G$ and the constant sheaf $\underline{\mathbb{R}}: (G,\unlhd) \to \text{Vect}_{\mathbb{R}}$. 
The graph $G$ consists of four vertices $u, v, w, x$ and three edges $\{u, w\}$, $\{v, w\}$, and $\{w, x\}$. The second row depicts the graphical correspondences of arbitrary elements $\mathbf{s} \in C^0(G; \underline{\mathbb{R}})$ and 
$\mathbf{t} \in C^1(G; \underline{\mathbb{R}})$ with the graphical representation of $\underline{\mathbb{R}}$. The third row depicts the cases where $\mathbf{s}, \mathbf{s}' \in C^0(G; \underline{\mathbb{R}})$, with $\mathbf{s} \in \Gamma(G; \underline{\mathbb{R}})$ and $\mathbf{s}' \notin \Gamma(G; \underline{\mathbb{R}})$.}
\label{Fig: Global sections}
\end{figure}

However, numerous collections of node signals of a graph may not behave as a global section of the sheaf, such as an eigenvector of $\mathbf{L}_{\mathcal{F}}^0$ corresponding to a nonzero eigenvalue. Figure~\ref{Fig: Global sections} illustrates a toy example of the constant sheaf 
$\underline{\mathbb{R}}: (G, \unlhd) \to \text{Vect}_{\mathbb{R}}$ defined on a graph, showing both a global section $\mathbf{s}$ and a $0$th sheaf cochain $\mathbf{s}'$ that does not constitute a global section.

\subsection{Alexandrov Topology and Local Sections}

In order to analyze the distribution of locally consistent components of a vector $\mathbf{s} \in C^0(G;\mathcal{F})$ associated with a cellular sheaf $\mathcal{F}$ on a graph $(G,\unlhd)$, it is necessary to introduce the concepts of the Alexandrov topology on $G$ and the local sections of $\mathcal{F}$. For completeness, we first provide a concise introduction to these notions in the general setting of posets~\cite{alexandroff1937diskrete}, and then specialize the discussion to their geometric interpretation on graphs.

Mathematically, a subset $U \subseteq P$ of a poset $(P,\leq)$ is called \textit{Alexandrov-open} if it satisfies the following condition: for any $x \in U$ and $y \in P$ with $x \leq y$, one has $y \in U$. The collection of all \textit{Alexandrov-open} subsets of $P$ forms a topology on $P$. Dually, a subset $C \subseteq P$ is called \textit{Alexandrov-closed} if, for any $x \in C$ and $y \in P$ with $y \leq x$, one has $y \in C$. In particular, a subset $U \subseteq P$ is Alexandrov-open if and only if its complement $P \setminus U$ is Alexandrov-closed. Moreover, the empty set $\emptyset$ and the entire set $P$ are both Alexandrov-open and Alexandrov-closed.

A particular type of Alexandrov-open set, called the \textit{star-open set} of a poset element, is of special interest. Formally, for a poset $(P,\leq)$ and an element $p \in P$, the \textit{star-open set} $U_p$ is defined as the minimal Alexandrov-open neighborhood containing $p$, that is,
\begin{equation*}
U_p = \{\, q \in P \mid p \leq q \,\},
\end{equation*}
where the minimality of $U_p$ means that the following property holds: if $p \in P$ and $U \subseteq P$ is an Alexandrov-neighborhood of $p$, then $U_p \subseteq U$.  It is well-known that the collection $\{\, U_p \mid p \in P \,\}$ of star-open sets in $P$ forms a basis for the Alexandrov topology, i.e., every Alexandrov-open subset of $P$ can be expressed as a union of a family of star-open sets in $P$.

In particular, for a graph $G = V \cup E$ with vertex and edge sets $V$ and $E$, the star open set $U_e$ corresponding to an edge $e \in E$ is precisely the singleton set $\{ e \}$, while the star open set $U_v$ associated with a vertex $v \in V$ consists of $v$ together with all edges incident to $v$. We recall that a \textit{subgraph} $H$ of a graph $G = V \cup E$ is a subset of $G$ satisfying the following property: for every edge $e \in E \cap H$ with endpoints $v, w \in V$, both $v$ and $w$ are contained in $H$. In particular, the following well-known proposition describes the relationship between the Alexandrov topology and the subgraph family of a graph. Figure \ref{Fig: Alexandrov topology} displays an example of a graph $G$ and some Alexandrov open subsets with respect to the face relation $\unlhd$ on $G$.

\begin{prop.}
A subset of a graph is a subgraph if and only if it is Alexandrov-closed.    
\end{prop.}  

With the Alexandrov topology on a given poset, one can define local sections of a cellular sheaf defined on that poset. Formally, let $\mathcal{F}: (G,\unlhd) \to \text{Vect}_{\mathbb{R}}$ be a cellular sheaf on a graph $(G,\unlhd)$ equipped with the Alexandrov topology, and let $U \subseteq G$ be an Alexandrov-open set. A tuple $(\mathbf{s}_p)_{p \in U}$ with $\mathbf{s}_p \in \mathcal{F}_p$ for each $p \in U$ is called a \textit{local section} of $\mathcal{F}$ on $U$ if it satisfies the following condition: for every $p, q \in U$ and $r \in U$ with $p \unlhd r$ and $q \unlhd r$, one has
\begin{equation}\label{Eq. Local section definition}
\mathcal{F}_{p\unlhd r}(\mathbf{s}_{p}) = \mathcal{F}_{q\unlhd r}(\mathbf{s}_{q}).
\end{equation}
The set of all local sections of the sheaf $\mathcal{F}$ on an Alexandrov-open set $U$ is denoted by $\mathcal{F}(U)$. It forms a subspace of the $\mathbb{R}$-vector space $\prod_{p \in U} \mathcal{F}_p$ with the canonical addition and scalar multiplication operations. Note that the spaces $\prod_{p \in U} \mathcal{F}_p$ and $\bigoplus_{p \in U} \mathcal{F}_p$ coincide when $G$ is finite, as in the case of the \textit{graphs} considered in this paper.

\begin{figure}[t]
\centering
\includegraphics[width=0.5\columnwidth]{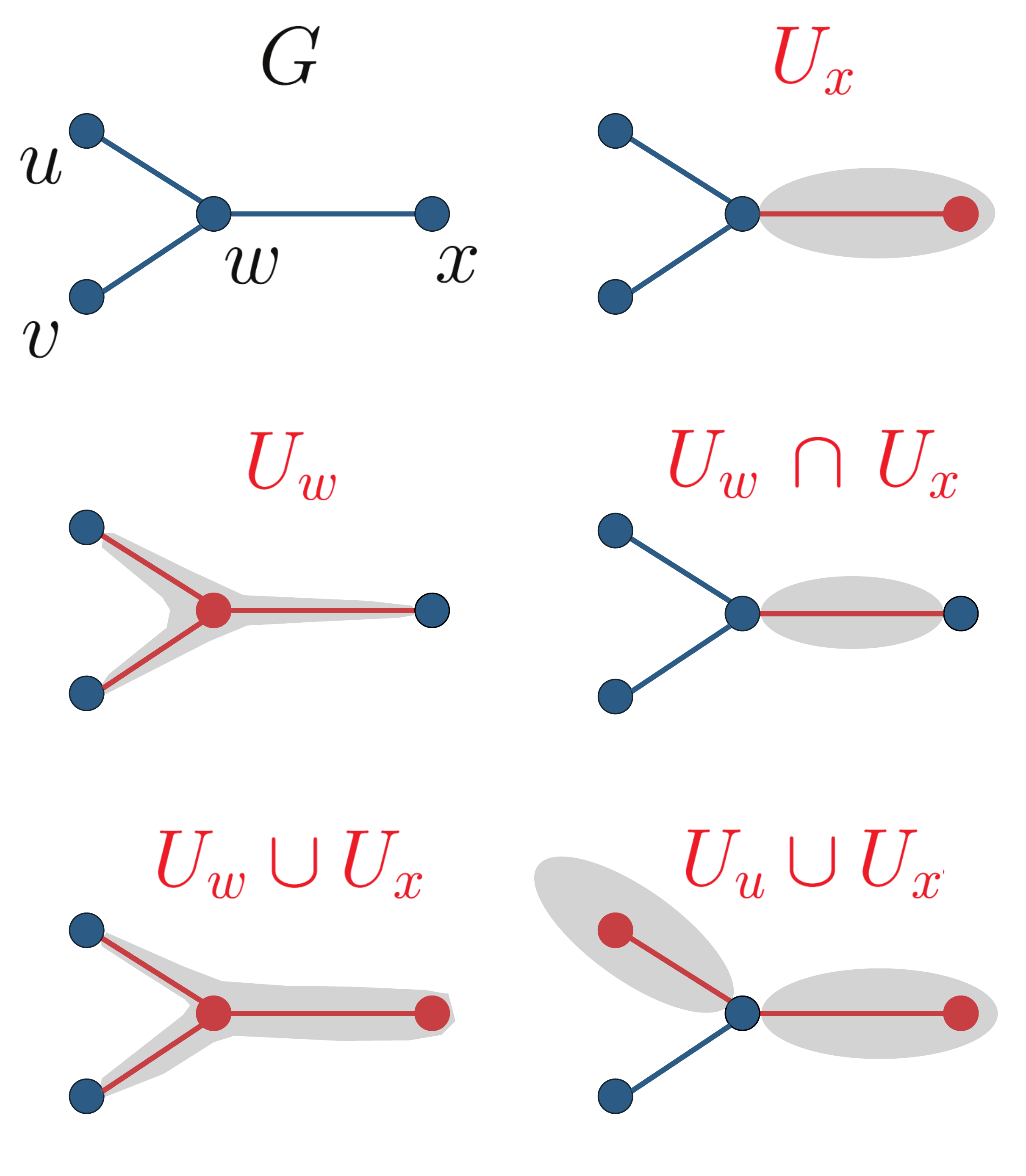} 
\caption{An example of a graph $(G,\unlhd)$ with vertex set $V=\{u,v,w,x\}$ and edge set $E=\{\{u,w\},\{v,w\},\{w,x\}\}$. With respect to the partial order $\unlhd$ on $G$, the Alexandrov open sets $U_x$, $U_w$, $U_w \cap U_x$, $U_w \cup U_x$, and $U_u \cup U_x$ are shown as the collections of vertices and edges covered by the shaded regions.}
\label{Fig: Alexandrov topology}
\end{figure}

\section{Framework and Methodology}
 
This section elaborates on the theoretical framework underlying the main results (a), (b), and (c), presenting their detailed formulations, proofs, and interpretations.

\subsection{Cellular Sheaf Modeling of Feature Attention Mechanisms in GNN Models}

During the training of GAT models, the feature or signal of each node is updated based on its neighboring nodes and the corresponding directional attention weights from the neighbors to that node. Mathematically, this information can be represented by the triple
\begin{equation}\label{Eq. The GAT triple}
(G, (\mathbf{s}_{v_i})_{i = 1}^n, \mathbf{W}),
\end{equation}
where $G = V \cup E$ denotes the underlying graph with vertex set $V = \{ v_1, v_2, \ldots, v_n \}$ and edge set $E$, $(\mathbf{s}_{v_i})_{i = 1}^n$ is the collection of $d$-dimensional signal (or feature) vectors on the nodes $v_i$, and $\mathbf{W} = (w_{ij}) \in \mathbb{R}^{n \times n}$ is the weight matrix of the attention mechanism. Each entry $w_{ij}$ represents the directional weight from node $v_i$ to node $v_j$ used in the feature aggregation process. Note that $\mathbf{W}$ is not necessarily symmetric, and $w_{ij} = 0$ whenever $\{v_i, v_j\} \notin E$. In particular, since the graph is assumed to be simple, we consider the case where $w_{ii} = 0$ for all $i$.

Specifically, for a given GAT triple defined in Equation~\eqref{Eq. The GAT triple} and an index $i \in \{1, 2, \ldots, n\}$, the core attention mechanism of the GAT, excluding the activation functions, aggregates the features $\mathbf{s}_{v_j}$ of the neighbors of $v_i$ as a linear combination
\begin{equation*}
\sum_{\substack{j \in \{1, \ldots, n\} \\ v_j \in \mathcal{N}(v_i)}} w_{ji} \cdot \mathbf{s}_{v_j},
\end{equation*}
where $\mathcal{N}(v_i)$ denotes the set of nodes in $G$ that are adjacent to $v_i$. In practice, the coefficients $w_{ji}$ are typically assumed to lie within the closed interval $[0,1]$ and satisfy the normalization condition $\sum_j w_{ji} = 1$. In general, a larger coefficient $w_{ji}$ indicates that the node feature $\mathbf{s}_{v_j}$ exerts a stronger influence on the representation of node $v_i$.

\begin{theorem}\label{Theorem: Main result 1}
Let $(G, (\mathbf{s}_{v_i})_{i = 1}^n, \mathbf{W})$ denote a triple defined as in Equation~\eqref{Eq. The GAT triple}, where $G = V \cup E$ and $\mathbf{s}_{v_i} \in \mathbb{R}^d$ for each $i \in \{1, \ldots, n\}$. Then, a cellular sheaf $\mathcal{F} : (G, \unlhd) \to \textup{\text{Vect}}_{\mathbb{R}}$ is defined as follows: (a) $\mathcal{F}_{v} := \mathbb{R}^d$ for each $v \in V$; (b) $\mathcal{F}_{e} := \mathbb{R}^d$ for each edge $e \in E$; and (c) $\mathcal{F}_{v_i \unlhd \{v_i, v_j\}}: \mathcal{F}_{v_i} \to \mathcal{F}_{\{v_i, v_j\}}$ is the scalar multiplication by $w_{ij}$ for every edge $\{v_i, v_j\} \in E$ with endpoints $v_i, v_j \in V$.
\end{theorem}
\begin{proof}
The stalk spaces $\mathcal{F}_{v}$ and $\mathcal{F}_{e}$ are vector spaces over $\mathbb{R}$ for every $v \in V$ and $e \in E$. Since each scalar multiplication $\mathbb{R}^d \to \mathbb{R}^d$, given by $\mathbf{x} \mapsto \lambda \mathbf{x}$ for a fixed $\lambda \in \mathbb{R}$, is $\mathbb{R}$-linear, the above construction defines a cellular sheaf $\mathcal{F}: (G, \unlhd) \to \textup{\text{Vect}}_{\mathbb{R}}$.
\end{proof}
More generally, suppose $\mathcal{G} : (G, \unlhd) \to \textup{\text{Vect}}_{\mathbb{R}}$ is a cellular sheaf on $G$ such that $\mathcal{G}_v = \mathbb{R}^d = \mathcal{G}_e$ for every $v \in V$ and $e \in E$. Then, the following assignment defines a more general cellular sheaf structure that extends the conventional attention aggregation formula:
\begin{equation*}
\sum_{\substack{j \in \{1, \ldots, n\} \\ v_j \in \mathcal{N}(v_i)}} 
w_{ji} \, \mathcal{G}_{v_j, \{ v_i, v_j \}} \, \mathbf{s}_{v_j},
\end{equation*}
which performs the feature aggregation through the sheaf-induced restriction maps $\mathcal{G}_{v_j, \{ v_i, v_j \}} \in \mathbb{R}^{d \times d}$ and the attention weights $w_{ji}$, allowing a more complex attention mechanism.

Figure~\ref{Fig: GAT to sheaf identification} provides a graphical representation of the cellular sheaf, defined in Theorem~\ref{Theorem: Main result 1}, associated with a GAT triple $(G, (\mathbf{s}_{v_i})_{i = 1}^n, \mathbf{W})$. In this example, the first image shows the attention weights originating from node $u$ to its neighbors, while the second image shows the attention weights from $u$'s neighbors to node $u$. The third image illustrates the induced cellular sheaf defined in Theorem~\ref{Theorem: Main result 1}.

\begin{figure*}[t]
\centering
\includegraphics[width=0.75\textwidth]{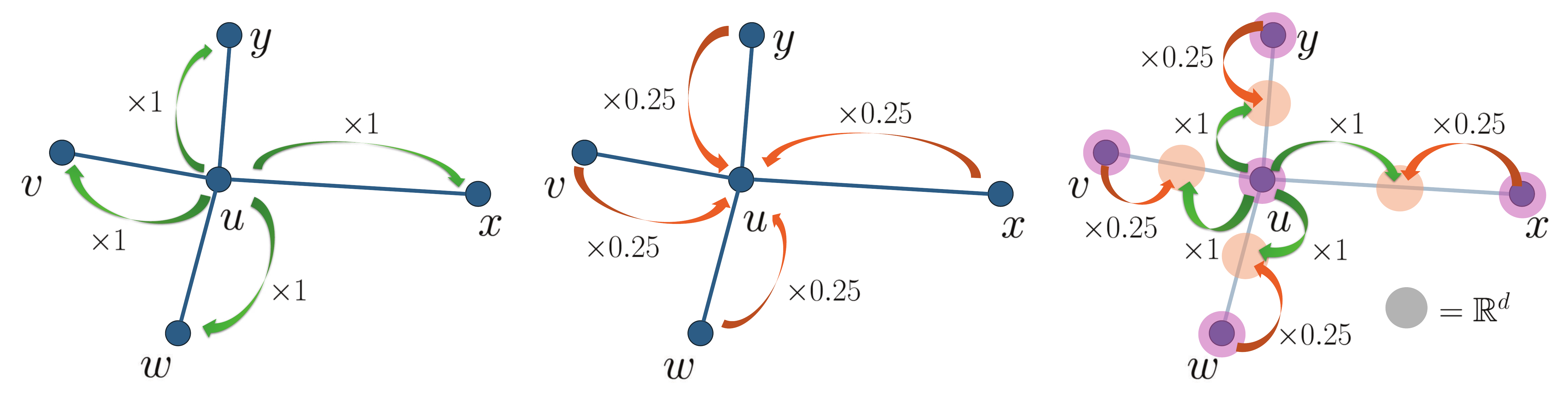} 
\caption{An illustrative example of the cellular sheaf identification of a GAT triple $(G, (\mathbf{s}_{u}, \mathbf{s}_{v}, \mathbf{s}_{w}, \mathbf{s}_{x}, \mathbf{s}_{y}), \mathbf{W} = (w_{ij}))$ with $i, j \in \{ u, v, w, x, y\}$ is shown, where the graph has the vertex set $V = \{u, v, w, x, y\}$ and the edge set $E = \{\{u, v\}, \{u, w\}, \{u, x\}, \{u, y\}\}$. The left image depicts the weights $w_{uv} = w_{uw} = w_{ux} = w_{uy} = 1$, while the middle image shows the weights $w_{vu} = w_{wu} = w_{xu} = w_{yu} = 0.25$. The right image illustrates the corresponding cellular sheaf defined as in Theorem~\ref{Theorem: Main result 1}.}
\label{Fig: GAT to sheaf identification}
\end{figure*}

\subsection{Sheaf Harmonic Structures for Graph Signals}
  
Given a cellular sheaf defined on a graph and a $0$th sheaf cochain representing a collection of node signals, we introduce the notion of harmonic sets to investigate the local harmonicity of node signals along edges. Specifically, consider the following setting. Let $\mathcal{F}: (G, \unlhd) \to \textup{\text{Vect}}_{\mathbb{R}}$ be a cellular sheaf, and let $\mathbf{C}^0 := \mathbf{C}_{\mathcal{F}}^0$ denote the $0$th sheaf coboundary matrix associated with $\mathcal{F}$. For a node signal vector $\mathbf{s} \in C^0(G;\mathcal{F})$, write $\mathbf{t} = (\mathbf{t}_e)_{e \in E} := \mathbf{C}^0 \mathbf{s}$. We then define the \textit{harmonic edges} and \textit{harmonic nodes} of the graph as follows.

\begin{def.}\label{Definition: Main result 2-1, harmonic sets}
An edge $e \in E$ with endpoints $v, w \in V$ is called a \textbf{harmonic edge} if $\mathbf{t}_e = \mathbf{0}$, that is, if $\mathcal{F}_{v, e} \mathbf{s}_{v} = \mathcal{F}_{w, e} \mathbf{s}_{w}$. Similarly, a node $v \in V$ is called a \textbf{harmonic node} if either $v$ has degree $0$, or there exists an edge $e \in E$ with $v \unlhd e$ such that $\mathbf{t}_e = \mathbf{0}$. The collections of all harmonic edges and harmonic nodes are denoted by $\textup{\text{Har}}_1(\mathbf{s})$ and $\textup{\text{Har}}_0(\mathbf{s})$, respectively. The \textbf{harmonic set} associated with $\mathbf{s}$ is defined the union $\textup{\text{Har}}(\mathbf{s}) := \textup{\text{Har}}_0(\mathbf{s}) \cup \textup{\text{Har}}_1(\mathbf{s})$.
\end{def.}

By the definition of the harmonic set associated with an $\mathbf{s} \in C^0(G;\mathcal{F})$, the following property holds: if $e \in \textup{\text{Har}}_1(\mathbf{s}) \subseteq \textup{\text{Har}}(\mathbf{s})$ and $v$ is an endpoint of $e$, then $v \in \textup{\text{Har}}_0(\mathbf{s}) \subseteq \textup{\text{Har}}(\mathbf{s})$. Furthermore, by definition, $\textup{\text{Har}}(\mathbf{s}) = G$ if and only if $\textup{\text{Har}}_1(\mathbf{s}) = E$. This implies that $\textup{\text{Har}}(\mathbf{s})$ forms a subgraph of $G$, and hence $\textup{\text{Har}}(\mathbf{s})$ is an Alexandrov-closed subset of $G$. In particular, $\textup{\text{Har}}(\mathbf{s}) = G$ if and only if $\mathbf{s}$ is a global section of $\mathcal{F}$ over $G$. These observations can be summarized in the following theorem.

\begin{theorem}\label{Theorem: Main result 2-2}
Let $\mathcal{F}$ be a cellular sheaf of $\mathbb{R}$-vector spaces on a graph $(G,\unlhd)$, and let $\mathbf{s} \in C^0(G;\mathcal{F})$ be a vector of node signals on $G$. Then,  $\textup{\text{Har}}(\mathbf{s}) = G$ if and only if $\mathbf{s} \in \Gamma(G;\mathcal{F})$. Furthermore, $\textup{\text{Har}}(\mathbf{s})$ is a subgraph consisting of a union of connected components of $G$ if and only if $\textup{\text{Har}}(\mathbf{s})$ is Alexandrov-open.    
\end{theorem}
\begin{proof}
Let $\mathbf{C}^0: C^0(G;\mathcal{F}) \to C^1(G;\mathcal{F})$ be the $0$th sheaf coboundary matrix of $\mathcal{F}$, and let $\mathbf{t} = (\mathbf{t}_e)_{e \in E} = \mathbf{C}^0 \mathbf{s}$. Note that $\mathbf{t}_e = \mathbf{0}$ if and only if $\mathcal{F}_{w \unlhd e}(\mathbf{s}_w) = \mathcal{F}_{v \unlhd e}(\mathbf{s}_v)$, where $v$ and $w$ are the endpoints of $e$. In particular, $\textup{\text{Har}}_1(\mathbf{s}) = E$ if and only if $\mathbf{s} \in \Gamma(G;\mathcal{F})$. Since $\textup{\text{Har}}(\mathbf{s}) = G$ if and only if $\textup{\text{Har}}_1(\mathbf{s}) = E$, the first statement follows.

For the second statement, suppose that $\textup{\text{Har}}(\mathbf{s})$ is a subgraph consisting of a union of connected components of $G$. Because $G$ is finite, the complement $G \setminus \textup{\text{Har}}(\mathbf{s})$ is a finite union of connected components of $G$. Because each connected component of $G$ is Alexandrov-closed, the set $G \setminus \textup{\text{Har}}(\mathbf{s})$ is also Alexandrov-closed, as desired. 

Conversely, suppose that $\textup{\text{Har}}(\mathbf{s})$ is Alexandrov-open. If $\textup{\text{Har}}(\mathbf{s}) = \emptyset$, then it is trivially the union of the empty family of connected components of $G$. If $\textup{\text{Har}}(\mathbf{s}) \neq \emptyset$, then there exists a node $v \in \textup{\text{Har}}_0(\mathbf{s})$. Let $H$ denote the connected component of $G$ that contains $v$. For every $w \in V \cap H$, there exists a finite sequence of nodes $v = v_0, v_1, \ldots, v_k = w$ in $H$ and edges $e_1 = \{v_0, v_1\}, e_2 = \{v_1, v_2\}, \ldots, e_k = \{v_{k-1}, v_k\}$ in $E \cap H$ connecting them. Let $l \in \{0, 1, \ldots, k\}$ be the largest index such that $v_l \in \textup{\text{Har}}_0(\mathbf{s})$. We claim that $l = k$. Suppose, for the sake of contradiction, that $l < k$. Then $v_l \in \textup{\text{Har}}_0(\mathbf{s})$ and $v_{l+1} \notin \textup{\text{Har}}_0(\mathbf{s})$. Because $\textup{\text{Har}}(\mathbf{s})$ is Alexandrov-open and contains $v_l$, the edge $e_{l+1} = \{v_l, v_{l+1}\}$ belong to $\textup{\text{Har}}_1(\mathbf{s})$, and this shows that $v_{l+1} \in \textup{\text{Har}}_0(\mathbf{s})$. This contradiction shows that $l = k$, and hence $w \in \textup{\text{Har}}_0(\mathbf{s})$ for every $w \in V \cap H$. Because $\textup{\text{Har}}(\mathbf{s})$ is Alexandrov-open, $\textup{\text{Har}}(\mathbf{s})$ contains the edges $e_1, e_2, \ldots, e_k$. This shows that $H \subseteq \textup{\text{Har}}(\mathbf{s})$, and we conclude that $\textup{\text{Har}}(\mathbf{s})$ is a subgraph consisting of a union of connected components of $G$.
\end{proof}

In particular, when the graph $G$ is connected, the harmonic set $\textup{\text{Har}}(\mathbf{s})$ is Alexandrov-open only if it is either empty or coincides with the entire graph. In the latter case, $\mathbf{s} \in \Gamma(G;\mathcal{F})$ is a global section on $G$. We state this result formally as the following corollary.

\begin{coro.}\label{Corollary: Main result 2-3}
Let $\mathcal{F}$ be a cellular sheaf of $\mathbb{R}$-vector spaces on a connected graph $(G,\unlhd)$, and let $\mathbf{s} \in C^0(G;\mathcal{F})$ be a vector of node signals on $G$. Then, $\textup{\text{Har}}(\mathbf{s})$ is Alexandrov-open if and only if $\textup{\text{Har}}(\mathbf{s}) = \emptyset$ or $\textup{\text{Har}}(\mathbf{s}) = G$. 
\end{coro.}

\subsection{Multi-Scale Sheaf Harmonic Structures}
In practical applications, numerical issues often cause the collection of node signals $\mathbf{s} = (\mathbf{s}_v)_{v \in V} \in C^0(G; \mathcal{F})$ on a graph $G = V \cup E$ to lack any pair $(v, w) \in V \times V$ with $e := \{ v, w \} \in E$ satisfying $\mathcal{F}_{v,e}\mathbf{s}_v = \mathcal{F}_{w,e}\mathbf{s}_w$. This results in an infeasible situation for detecting local alignments based solely on the node features. To address this problem, we introduce a multiscale framework, presented in the following paragraphs, for detecting local alignments of node features.

Specifically, let $\mathcal{F}:(G,\unlhd)\to \text{Vect}_{\mathbb{R}}$ be a cellular sheaf defined on a graph $G = V \cup E$. For every vector $\mathbf{s}=(\mathbf{s}_v)_{v \in V} \in C^0(G;\mathcal{F})$ of node signals $\mathbf{s}_v \in \mathcal{F}_v$, we define the \textit{sheaf norm} of $\mathbf{s}$, denoted by $\Vert \mathbf{s} \Vert_{\mathcal{F}}$, as
\begin{equation}\label{Eq. sheaf norm}
\Vert \mathbf{s} \Vert_{\mathcal{F}} := \bigl\Vert \mathbf{C}^0 \mathbf{s} \bigr\Vert_2,
\end{equation}
where $\mathbf{C}^0 := \mathbf{C}_{\mathcal{F}}^0: C^0(G;\mathcal{F}) \to C^1(G;\mathcal{F})$ is the $0$th sheaf coboundary matrix. Then, $\Vert \cdot \Vert_{\mathcal{F}}$ defines a seminorm on $C^0(G;\mathcal{F})$, since $\mathbf{C}^0$ is $\mathbb{R}$-linear and $\Vert \cdot \Vert_2$ is a norm on $C^1(G;\mathcal{F})$. 

\begin{remark}\label{Remark: when a sheaf norm forms a norm}
The sheaf norm $\Vert \cdot \Vert_{\mathcal{F}}$ is a norm on $C^0(G;\mathcal{F})$ if and only if the coboundary matrix $\mathbf{C}^0$ is injective. Specifically, $\Vert \mathbf{C}^0 \mathbf{s} \Vert_2 = 0$ implies $\mathbf{C}^0 \mathbf{s} = \mathbf{0}$, which in turn shows that $\mathbf{s} = \mathbf{0}$ whenever $\mathbf{C}^0$ is injective. Conversely, if $\Vert \mathbf{s} \Vert_{\mathcal{F}}$ is a norm and $\mathbf{s} \in \ker(\mathbf{C}^0)$, then $\Vert \mathbf{s} \Vert_{\mathcal{F}} = \Vert \mathbf{C}^0 \mathbf{s} \Vert_2 = 0$, and this forces that $\mathbf{s} = \mathbf{0}$.
\end{remark}

\begin{remark}
From the perspective of energy, the sheaf norm in \eqref{Eq. sheaf norm} can be interpreted as the square root of a \textbf{quadratic potential function} of sheaf signals~\cite{zhao2025asynchronousnonlinearsheafdiffusion}. 
A related energy formulation based on the \emph{normalized sheaf Laplacian} is considered in sheaf diffusion models~\cite{bodnar2022neural}.
\end{remark}

The vector $\mathbf{C}^0 \mathbf{s}$, together with its norm $\Vert \mathbf{C}^0 \mathbf{s} \Vert_2$, provides a quantitative measure of the harmonicity of the node signals $\mathbf{s}_v$ on the graph. Notably, by Remark~\ref{Remark: when a sheaf norm forms a norm}, we have $\mathbf{s} \in \Gamma(G; \mathcal{F})$ whenever $\Vert \mathbf{s} \Vert_{\mathcal{F}} = 0$. This observation motivates a TDA-based multiscale framework that builds upon the substructures of $G$ constructed according to various upper bounds of the sheaf norm values.

\begin{def.}
Let $(\mathcal{F}, G, V, E, \mathbf{C^0}, \mathbf{s}, \mathbf{t})$ denote the data defined above, and let $\epsilon \geq 0$ be a nonnegative real number. An edge $e \in E$ is called an $\epsilon$\textbf{-harmonic edge} if $\Vert \mathbf{t}_e \Vert_2 \leq \epsilon$, and a node $v \in V$ is called an $\epsilon$\textbf{-harmonic node} if either $v$ has degree $0$, or there exists an edge $e \in E$ with $v \unlhd e$ such that $\Vert \mathbf{t}_e \Vert_2 \leq \epsilon$. The collections of all $\epsilon$-harmonic edges and harmonic nodes are denoted by $\textup{\text{Har}}_1^\epsilon(\mathbf{s})$ and $\textup{\text{Har}}_0^\epsilon(\mathbf{s})$, respectively. The \textbf{harmonic set} associated with $\mathbf{s}$ is defined the union $\textup{\text{Har}}^\epsilon(\mathbf{s}) := \textup{\text{Har}}_0^\epsilon(\mathbf{s}) \cup \textup{\text{Har}}_1^\epsilon(\mathbf{s})$.     
\end{def.}

Similar to the proof that any harmonic set forms a subgraph, if $e \in \textup{\text{Har}}_1^\epsilon(\mathbf{s})$ is an edge with endpoints $v$ and $w$ in $V$, then $v, w \in \textup{\text{Har}}_0^\epsilon(\mathbf{s})$, implying that every $\epsilon$-harmonic set constitutes a subgraph of $G$. In particular, inspired by TDA frameworks, we introduce the notion of a \textit{filtration of harmonic sets} (see Corollary \ref{Corollary: Main result 3-2}), which is established in the following theorem.

\begin{theorem}\label{Theorem: Main result 3-1}
Let $\mathcal{F}$ be a cellular sheaf of $\mathbb{R}$-vector spaces on a connected graph $(G,\unlhd)$, and let $\mathbf{s} \in C^0(G;\mathcal{F})$ be a vector of node signals on $G$. If $\epsilon_1 \leq \epsilon_2$ are nonnegative real numbers, then
\begin{equation*}
\textup{\text{Har}}^{\epsilon_1}(\mathbf{s}) \subseteq \textup{\text{Har}}^{\epsilon_2}(\mathbf{s}).    
\end{equation*}
\end{theorem}
\begin{proof}
It is equivalent to show that $\textup{\text{Har}}_0^{\epsilon_1}(\mathbf{s}) \subseteq \textup{\text{Har}}_0^{\epsilon_2}(\mathbf{s})$ and $\textup{\text{Har}}^{\epsilon_1}_1(\mathbf{s}) \subseteq \textup{\text{Har}}_1^{\epsilon_2}(\mathbf{s})$. Specifically, if $e \in \textup{\text{Har}}^{\epsilon_1}_1(\mathbf{s})$, then $\Vert \mathbf{t}_e \Vert_2 \leq \epsilon_1 \leq \epsilon_2$, and this shows that $e \in \textup{\text{Har}}^{\epsilon_2}_1(\mathbf{s})$. On the other hand, suppose $v \in \textup{\text{Har}}^{\epsilon_1}_0(\mathbf{s})$. If $v$ has degree $0$, then $v$ is also an element in $\textup{\text{Har}}^{\epsilon_2}_0(\mathbf{s})$. Otherwise, $v$ is an endpoint for some edge $e \in \textup{\text{Har}}^{\epsilon_1}_1(\mathbf{s})$. Because $\textup{\text{Har}}^{\epsilon_1}_1(\mathbf{s}) \subseteq \textup{\text{Har}}^{\epsilon_2}_1(\mathbf{s})$, we deduce that $v \in \textup{\text{Har}}^{\epsilon_2}_0(\mathbf{s})$.
\end{proof}

\begin{coro.}\label{Corollary: Main result 3-2}
Let $\mathcal{F}$, $(G, \unlhd)$, and $\mathbf{s} \in C^0(G; \mathcal{F})$ be defined as in Theorem~\ref{Theorem: Main result 3-1}. Then, any increasing sequence of nonnegative real numbers $0 \leq \epsilon_1 < \epsilon_2 < \cdots < \epsilon_r$ induces the following filtration of subgraphs of $G$:
\begin{equation}\label{Eq. Filtration-1}
\textup{\text{Har}}^{\epsilon_1}(\mathbf{s}) \subseteq \textup{\text{Har}}^{\epsilon_2}(\mathbf{s}) \subseteq \cdots\subseteq \textup{\text{Har}}^{\epsilon_r}(\mathbf{s}).
\end{equation}
In particular, the persistence barcode of this filtration can be computed.
\end{coro.}

Furthermore, from the proof of Theorem~\ref{Theorem: Main result 3-1}, we observe that any increasing sequence of nonnegative real numbers $0 \leq \epsilon_1 < \epsilon_2 < \cdots < \epsilon_r$ induces the following filtrations. First, the filtration
\begin{equation}\label{Eq. Filtration-2}
\textup{\text{Har}}^{\epsilon_1}_0(\mathbf{s}) \subseteq \textup{\text{Har}}^{\epsilon_2}_0(\mathbf{s}) \subseteq \cdots\subseteq \textup{\text{Har}}^{\epsilon_r}_0(\mathbf{s})
\end{equation}
consists of subgraphs (i.e., Alexandrov-closed subsets) of $G$. On the other hand, the filtration
\begin{equation}\label{Eq. Filtration-3}
\textup{\text{Har}}^{\epsilon_1}_1(\mathbf{s}) \subseteq \textup{\text{Har}}^{\epsilon_2}_1(\mathbf{s}) \subseteq \cdots\subseteq \textup{\text{Har}}^{\epsilon_r}_1(\mathbf{s})
\end{equation}
is composed of Alexandrov-open subsets of $G$. Although the $1$st harmonic subsets are not Alexandrov-closed in general, which causes an inconvenience in the computation of persistence barcode, this issue can be resolved by considering their Alexandrov closures. Specifically, the sequence
\begin{equation}\label{Eq. Filtration-4}
\text{cl}(\textup{\text{Har}}^{\epsilon_1}_1(\mathbf{s})) \subseteq \text{cl}(\textup{\text{Har}}^{\epsilon_2}_1(\mathbf{s})) \subseteq \cdots\subseteq \text{cl}(\textup{\text{Har}}^{\epsilon_r}_1(\mathbf{s})),
\end{equation}
forms a filtration of subgraphs of $G$, providing a foundation for the computation of persistence barcode. 

\subsection{Future Work}
The proposed sheaf-theoretic and TDA-based framework opens several promising directions for future research in geometric and topological deep learning. From a theoretical perspective, extending the current formalism to higher-dimensional cellular structures, such as simplicial representations, could further generalize the harmonic analysis of features beyond pairwise relations. Furthermore, a more thorough theoretical analysis of the relationship between the proposed sheaf-based harmonic analysis and the learning behavior of the GDL or TDL models remains to be explored. On the computational side, incorporating these topological representations into the training and evaluation phases of graph neural networks may enable more interpretable and topology-aware learning paradigms. Future work will also investigate the use of sheaf-based filtrations as diagnostic tools for understanding oversmoothing, heterophily, and hierarchical feature propagation in complex network architectures.

\section{Conclusion}
This work establishes a theoretical framework for modeling complex networks and graph neural models through cellular sheaf theory and topological data analysis (TDA). By interpreting attention mechanisms as cellular sheaves, it provides an algebraic–topological basis for studying local feature alignment and global consistency. The introduced notions of harmonic node and edge sets, together with a TDA-based multiscale filtration framework, enable quantitative and hierarchical analyses of signal coherence and substructures, grounded in the local harmonicity of signals and the geometry of the underlying domain. Overall, the proposed approach unifies sheaf theory, graph signal analysis, and persistent topology, offering a new mathematical lens for understanding complex networks and attention-based architectures in geometric and topological deep learning.

\section{Acknowledgments}
This work was supported by the National Science and Technology Council (NSTC) of Taiwan under Grant No. NSTC 114-2811-M-008-069 during the author’s postdoctoral research at the Department of Mathematics, National Central University (NCU). The author gratefully acknowledges Dr. John M. Hong of NCU and the reviewers for their invaluable and constructive suggestions, as well as Dr. Hong’s generous support of this work. The author employed OpenAI’s ChatGPT-5 to improve the grammatical accuracy and linguistic fluency of the manuscript and remains solely responsible for the technical and mathematical content and its interpretation.

\bibliography{aaai2026}

\end{document}